\documentclass{article}
\usepackage{kr}

\usepackage{soul}
\usepackage{url}
\usepackage[hidelinks]{hyperref}
\usepackage[utf8]{inputenc}
\usepackage[small]{caption}
\usepackage{booktabs}
\usepackage{algorithm}
\usepackage{algorithmic}
\usepackage{helvet,times,courier}
\usepackage{amsmath}%[fleqn]\setlength{\mathindent}{4pt}
\usepackage{amssymb}
\usepackage{amsthm}
\theoremstyle{definition}
\newtheorem{definition}{Definition}
\newtheorem{theorem}{Theorem}
\newtheorem{proposition}{Proposition}

\newtheorem{corollary}{Corollary}
\newtheorem{example}{Example}
\usepackage{graphicx}
\usepackage{tikz}
\usetikzlibrary{positioning,arrows,calc,shapes.multipart}

\newcommand{\vdashv}{%
  \vdash\mathrel{\mkern-9mu}\dashv
}

\newcommand{\refero}[1]{#1}
\newcommand{\refer}[2]{#1(#2)}

\newcommand{\act}{\ensuremath{a}}
\newcommand{\pre}{\ensuremath{\mathit{pre}}}
\newcommand{\eff}{\ensuremath{\mathit{eff}}}
\newcommand{\efp}{\ensuremath{\mathit{add}}}
\newcommand{\efm}{\ensuremath{\mathit{del}}}
\newcommand{\dur}{\ensuremath{\mathit{dur}}}
\newcommand{\var}{\ensuremath{v}}
\newcommand{\task}{\ensuremath{\mathcal{T}}}
\newcommand{\vars}{\ensuremath{V}}
\newcommand{\acts}{\ensuremath{A}}
\newcommand{\state}{\ensuremath{s}}
\newcommand{\aux}{\ensuremath{x}}
\newcommand{\auxu}{\ensuremath{X}}
\newcommand{\lit}{\ensuremath{\ell}}
\newcommand{\init}{\ensuremath{\state_0}}
\newcommand{\goal}{\ensuremath{\state_\ast}}
\newcommand{\plan}{\ensuremath{\mathcal{P}}}
\newcommand{\stamp}{\ensuremath{t}}
\newcommand{\seq}{\ensuremath{P}}
\newcommand{\effq}{\ensuremath{E}}
\newcommand{\auxs}[1]{\ensuremath{{#1^\aux}}}
\newcommand{\auxz}[1]{\ensuremath{{#1^0}}}
\newcommand{\seqi}[1]{\ensuremath{\seq[#1]}}
\newcommand{\seqj}[1]{\ensuremath{\seq'[#1]}}
\newcommand{\stampi}[1]{\ensuremath{\stamp[#1]}}
\newcommand{\statei}[1]{\ensuremath{\state[#1]}}
\newcommand{\effi}[1]{\ensuremath{\effq[#1]}}
\newcommand{\ups}[1]{\ensuremath{{#1}^{\vdash}}}
\newcommand{\upd}[1]{\ensuremath{{#1}^{\vdashv}}}
\newcommand{\upe}[1]{\ensuremath{{#1}^{\dashv}}}
\newcommand{\dura}[1]{\ensuremath{\dur(#1)}}
\newcommand{\pres}[1]{\ensuremath{\ups{\pre}(#1)}}
\newcommand{\pred}[1]{\ensuremath{\upd{\pre}(#1)}}
\newcommand{\pree}[1]{\ensuremath{\upe{\pre}(#1)}}
\newcommand{\effs}[1]{\ensuremath{\ups{\eff}(#1)}}
\newcommand{\effe}[1]{\ensuremath{\upe{\eff}(#1)}}

\newcommand{\effsp}[1]{\ensuremath{\ups{\mathit{add}}(#1)}}%{\ensuremath{\ups{\eff}_\positive(#1)}}
\newcommand{\effsm}[1]{\ensuremath{\ups{\mathit{del}}(#1)}}%{\ensuremath{\ups{\eff}_\negative(#1)}}
\newcommand{\effep}[1]{\ensuremath{\upe{\mathit{add}}(#1)}}%{\ensuremath{\upe{\eff}_\positive(#1)}}
\newcommand{\effem}[1]{\ensuremath{\upe{\mathit{del}}(#1)}}%{\ensuremath{\upe{\eff}_\negative(#1)}}
\newcommand{\ind}[1]{\ensuremath{\iota(#1)}}
\newcommand{\auxa}[1]{\ensuremath{\aux(#1)}}
\newcommand{\auxv}[1]{\ensuremath{\aux_{#1}}}
\newcommand{\comp}[2]{\ensuremath{#1 \circ #2}}
\newcommand{\refined}{\ensuremath{\theta}}
\newcommand{\refine}[1]{\ensuremath{\refined_#1}}
\newcommand{\refinex}[3]{\ensuremath{\refine#1^#2#3}}
\newcommand{\deltax}[2]{\ensuremath{\varepsilon_#1#2}}

\usepackage{array}
\newcolumntype{?}[1]{!{\vrule width #1}}

\newcommand{\xalign}{\:\!}

\title{Enhancing Temporal Planning % Domains
by Sequential Macro-actions (Extended Version)}

\author{%
Marco De Bortoli$^1$\and
Lukáš Chrpa$^2$\and
Martin Gebser$^{1,3}$\and
Gerald Steinbauer-Wagner$^1$ \\
\affiliations
$^1$Graz University of Technology\\
$^2$Czech Technical University in Prague\\
$^3$University of Klagenfurt\\
\emails
\{mbortoli, mgebser, steinbauer\}@ist.tugraz.at,
chrpaluk@cvut.cz% ,
}

\begin{document}

\maketitle

\begin{abstract}

Temporal planning is an extension of classical planning involving concurrent execution of actions and alignment with temporal constraints. Durative actions along with invariants allow for modeling
domains in which multiple agents operate in parallel on shared resources. Hence, it is often important to avoid resource conflicts, where temporal constraints establish the consistency of concurrent actions and events. Unfortunately, the performance of temporal planning engines tends to sharply deteriorate when the number of agents and objects in a domain gets large. A possible remedy is to use macro-actions that are well-studied in the context of classical planning. In temporal planning settings, however, introducing macro-actions is significantly more challenging when the concurrent execution of actions and shared use of resources, provided the compliance to temporal constraints, should not be suppressed entirely.
Our work contributes a general concept of sequential temporal macro-actions that guarantees the applicability of obtained plans, i.e., the sequence of original actions encapsulated by a macro-action is always executable. We apply our approach to several temporal planners and domains, stemming from the International Planning Competition and RoboCup Logistics League. Our experiments yield improvements in terms of obtained satisficing plans as well as plan quality for the majority of tested planners and domains.

\end{abstract}

\section{Introduction}

%Logistics domains are important for a wide range of practical problems such as on-demand transport, parcel delivery, cargo delivery, shuttle services to mention some. In a nutshell, logistics problems involve delivering objects by vehicles while specified constraints (e.g., vehicle capacity, fuel) are satisfied. Specialised approaches, for instance, those addressing Vehicle Routing Problems~\cite{VRP}, can tackle specific classes of logistics problems, however, they lack flexibility to deal with changes in the domain specification. Logistics domains can also be tackled by domain-independent planning engines that addresses the flexibility issue. Many logistics domains became popular in the planning community since they serve as benchmarks for domain-independent planners~\cite{VallatiCM18}. 

%However, in logistics domains, it is usually important to minimize (total) time of delivery (of all objects) and hence temporal settings might be necessary. Temporal Planning provides a machinery to tackle logistics problems. In particular, we can model logistics domains in PDDL 2.1~\cite{foxlon03a} and use some of the off-the-shelf planners (e.g. Optic~\cite{ICAPS124699}) to generate plans. However, with a higher number of vehicles and/or objects to be delivered (temporal) planning engines tend to struggle to generate plans that limits their possible use for more practical problems. 

Temporal planning is a framework capable of dealing with concurrent actions and timing requirements, providing an intuitive syntax for representing planning domains, such as PDDL 2.1 \cite{foxlon03a}, together with off-the-shelf planners, e.g., Optic \cite{ICAPS124699}, to generate plans. As an extension of classical planning, temporal planning offers modeling support for durative actions, their concurrent execution, and the management of temporal constraints. Logistics domains are prominent examples in which such timing information matters, e.g., for planning on-demand transport and delivery, cargo shipment, shuttle services, or just-in-time production, to mention some application areas.  However, with larger numbers of tasks and/or resources to operate and synchronize, the performance of (temporal) planning engines tends to sharply deteriorate, which limits their usability for % challenging 
practical problem solving. % The same issue applies for other types of domains as soon as the size of the problem starts growing.

As a possible remedy for the scalability issue, in this paper, we develop a general approach to introduce macro-actions that % , in a nutshell, are actions encapsulating 
encapsulate sequences of ordinary actions.
Macro-actions are widely studied in classical planning % ; see, e.g., 
\cite{botea2005macro,coles2007marvin,chrpa2014mum}, capable of expressing composite behaviors, % useful activities related to logistics domains,
e.g.,
frequent activities like \emph{load-move-unload} sequences in transport domains,
% , which delivers an object from its location of origin to its goal location in a single step.
and proved to enable considerable speed-ups of classical planning engines \cite{chrpa2014mum,ChrpaV22}.
In temporal planning settings, however, the conception of macro-actions is significantly more challenging. Specifically, PDDL 2.1 limits the points in time to which preconditions and effects of durative actions can be linked to an action's start, end, or invariants over the entire duration. % , or the entire action duration.
Hence, a macro-action encapsulating a sequence of durative actions cannot exactly specify the times at which preconditions and effects associated with its constituent actions take place, which puts either the plan validity or concurrency into question.
To give an example, for a \emph{load-move-unload} macro-action, PDDL 2.1 does not permit to accurately capture when an agent leaves the ``loading'' location and arrives at the ``unloading'' location, so that other agents risk collisions or need to avoid both locations for the entire duration.
% or when the ``unloading'' location has to be free.

In this paper, we provide a general concept of sequential temporal macro-actions (Section~\ref{macrodefinition}), i.e., macros % composed from 
encapsulating sequences of durative actions,
where preconditions, invariants, and effects are assembled in a fine-grained way to enable concurrent execution when it does not compromise the macro-action applicability.
% and investigate in what conditions such macro-actions can be safely used in the planning process, that is, when macro-actions in plans are unfolded the resulting plans are still valid.
Sequential temporal macro-actions can be particularly advantageous in logistics domains, where transport and delivery tasks are often accomplished by specific sequences of actions like, e.g., \emph{load-move-unload}.
We thus apply and evaluate our approach using three state-of-the-art planners
on four domains (three of which are logistics-related) from the International Planning Competition as well as the RoboCup Logistics League (Section~\ref{evaluation}),
obtaining improvements in terms of coverage and in
some cases also the quality of plans computed in the same time limit with or
without macro-actions, respectively.
Finally, we discuss related work (Section~\ref{related}) and provide conclusions along with directions for future work (Section~\ref{conclusion}).

%The paper is structured as follows. In the next Section, a formalization of the a temporal planning problem is given. The main contribution of the paper is described in Section 3, where our methodology is formalized. Then, in Section 4 the application of the presented procedure is shown over a toy example, in order to better fix the idea. Section 5 follows, where the macro procedure is applied and evaluated over two different domains, after discussing their properties. In Section 6, an overview of the related literature over macro-actions, for both classical and temporal planning, is presented. Finally, in Section 7 we draw the conclusions of the paper and present the future work.

\section{Background}\label{background}

In temporal planning, states of the environment are represented by sets of propositions (or atoms). In contrast to classical planning, states of the environment are modified by events that are triggered by application of durative actions. 
A durative \emph{action}~$\act$ is defined by a \emph{duration} $\dura{\act}\in\mathbb{R}^+$, 
the sets $\pres{\act}$, $\pred{\act}$, and $\pree{\act}$ of atoms
specifying \emph{preconditions} that must hold at the start,
as invariants during, or at the end of the action~$\act$, respectively,
as well as the sets $\effs{\act}$ and $\effe{\act}$ of literals determining
\emph{effects} that apply at the start or at the end of~$\act$. 
%For $\effs{\act}$ $(\effe{\act})$, 
Let $\effsp{\act}$ (or $\effep{\act}$) and $\effsm{\act}$ (or $\effem{\act}$) denote the sets of atoms occurring as positive or negative literals, respectively, in $\effs{\act}$ (or $\effe{\act}$).
%For $\effa{\act}=\effs{\act}$ or $\effa{\act}=\effe{\act}$, let $\effp{\act}$ and $\effm{\act}$
%denote the sets of atoms occurring as positive or negative literals, respectively, in $\effa{\act}$.
% Without loss of generality, 
W.l.o.g.\ we assume that
$\effsm{\act}\cap(\effsp{\act}\cup\pred{\act})=\emptyset$
and
$\effem{\act}\cap\effep{\act}=\emptyset$, given that the action~$\act$
would be inapplicable otherwise.
Note that positive effects $\effsp{\act}$ and $\effep{\act}$
% make atoms true and 
enable preconditions required by actions,
while $\effsm{\act}$ and $\effem{\act}$ inhibit actions.
% falsify them and do thus inhibit actions.

A temporal \emph{planning task}
$\task = (\vars,\acts,\init,\goal)$
consists of a set $\vars$ of atoms,
a set $\acts$ of actions such that
$\pres{\act}\cup\pred{\act}\cup\pree{\act}\cup
 \effsp{\act}\cup\effsm{\act}\cup\effep{\act}\cup\effem{\act}
 \subseteq \vars$
for each $\act\in\acts$,
a (total) \emph{initial state} $\init\subseteq\vars$
represented by the % set of 
atoms that hold,
and a (partial) \emph{goal} $\goal\subseteq\nolinebreak\vars$
inducing % set $\{\state\subseteq V\mid\goal\subseteq\state\}$ of goal states.
the goal states $\{\state\subseteq V\mid\goal\subseteq\state\}$.
A \emph{plan}~$\plan$ for $\task$ is a finite set
of \emph{time-stamped actions} $(\stamp,\act)$, where $\stamp\in\mathbb{R}^+$ is 
a starting time for the action $\act\in\acts$,
such that $\stamp_1\neq\stamp_2$, $\stamp_1+\dura{\act_1}\neq\stamp_2$, and
$\stamp_1+\dura{\act_1}\neq\stamp_2+\dura{\act_2}$ for any distinct time-stamped actions $(\stamp_1,\act_1)\nolinebreak\in\plan,\linebreak[1](\stamp_2,\act_2)\in\plan$.
% $(\stamp_1,\act_1)\in\plan$ and $(\stamp_2,\act_2)\in\plan$.
The condition of mutually disjoint starting and ending times % of time-stamped actions
is in line with the ``no moving targets'' requirement of PDDL 2.1 % \cite{foxlon03a}
and also adopted by Temporal Fast Downward \cite{eymaro09a} for enabling a sequential processing of action effects and state updates.
That is, a plan $\plan$ % of size $n$
corresponds to a sequence~$\seq$
with $2|\plan|+1$ \emph{time-stamped states}
of the form 
$\seqi{i}=(\stampi{i},\statei{i},\effi{i})$,
where
$\seqi{0}=(0,\init,\emptyset)$ and
successors are inductively
defined as follows 
for $1\leq i\leq 2|\plan|$:%
\begin{align*}
\stampi{i}={}&
\begin{array}[t]{@{}l@{}l@{}}
  \min\big(
      \{\stamp
       \mid
       (\stamp,\act)\in\plan,\stampi{i-1}<\stamp\}
     \\\xalign{}\cup
      \{\stamp
      \mid
      (\stamp,\pred{\act},\pree{\act},\effe{\act})\in\effi{i-1}\}\big)
\text{;}
\end{array}
\\
\statei{i}={}&
 \begin{cases}
  \effsp{\act}\cup
  (\statei{i-1}\setminus\effsm{\act})
  \\ \quad \text{if } (\stampi{i},\act)\in\plan \text{, or}
  \\
  \effep{\act}\cup
  (\statei{i-1}\setminus\effem{\act})
  \\ \quad \text{if } (\stampi{i},\pred{\act},\pree{\act},\effe{\act})\in\effi {i-1}\text{;}
 \end{cases}
\\
\effi{i}={}&
 \begin{cases}
  \effi{i-1}\\ {}\cup
  \{(\stampi{i}+\dura{\act},\pred{\act},\pree{\act},\effe{\act})\}
  \\ \quad  \text{if } (\stampi{i},\act)\in\plan \text{, or}
  \\
  \effi{i-1}\setminus
  \{(\stampi{i},\pred{\act},\pree{\act},\effe{\act})\}
  \\ \quad  \text{if } (\stampi{i},\pred{\act},\pree{\act},\effe{\act})\in\effi{i-1}\text{.}
 \end{cases}
\end{align*}
The sets $\effi{i}$ gather \emph{scheduled effects} for actions that
already started but have not yet ended at time stamp $\stampi{i}$. % and 
We
call $\seqi{i}$ a \emph{starting} or \emph{ending event}
for $\act$, respectively, if $(\stampi{i},\act)\in\plan$ or
$(\stampi{i},\pred{\act},\pree{\act},\effe{\act})\in\effi{i-1}$, % .
and by $\ind{\stampi{i}}=i$
we refer to the sequence position associated with a time stamp $\stampi{i}$.%
\footnote{%
For simplicity of notations and as the considered plan $\plan$ is always clear from the context, we skip
annotations for $\plan$ when writing
$\stampi{i}$, $\statei{i}$, $\effi{i}$, and
$\ind{\stamp}$ to refer to time stamps, states,
scheduled effects, or sequence positions of time-stamped
states $\seqi{i}$ for~$\plan$.}
Moreover, the plan~$\plan$ % for $\task$ 
% as well as its corresponding sequence~$\seq$ are
is \emph{consistent} if, for each $1\leq i\leq 2|\plan|$,
we have that
$\bigcup_{(\stamp,\pred{\act},\pree{\act},\effe{\act})\in\effi{i}}\pred{\act}\subseteq \statei{i}$ and
$\pres{\act}\subseteq\statei{i-1}$ or
$\pree{\act}\subseteq\statei{i-1}$
in case $\seqi{i}$ is a starting or ending event for $\act$,
respectively.
If $\plan$ is consistent and $\goal\subseteq\statei{2|\plan|}$,
we call the plan $\plan$ a \emph{solution} for% the planning task
~$\task$.

\section{Sequential Macro-actions}\label{macrodefinition}

In contrast to classical planning, where actions are viewed
as instantaneous events and modeling their sequential execution
by a macro-action boils down to accumulating preconditions and effects
independently of other actions, concurrent actions
must be taken into account for temporal planning, which makes it
non-trivial to guarantee the applicability of plans with macro-actions.
For instance, when the sequence of a move and a delivery action with the mutually exclusive invariants that a robot must not deliver while moving, and vice versa, should be turned into a macro, simply adding
up effects and invariants would render the macro-action inapplicable,
as the processes of moving and delivering are both triggered % at the start
but neither of them should be in progress according to the combined invariant.
Then, it may appear tempting to dismiss invariants in favor of additional preconditions modeling mutex locks, thus ensuring that atoms affected by a macro-action are not accessed by other
concurrent actions.
While this approach rules out any interferences, it is also disruptive
in the sense that plans get partitioned into episodes:
a macro-action takes exclusive control of its mutex atoms, and only
before starting or after ending the macro-action, other actions may again access the mutex atoms (in parallel).
To see this, consider a macro-action of moving to the sink and washing plates, requiring some dirty plates to be on hold,
and a second macro-action that consists of collecting dirty plates from the table and carrying them to the sink.
Then, if the precondition or effect, respectively, on the availability of dirty
plates were mutually exclusive, it would rule out
that one robot brings dirty plates from the table while another
is washing at the sink, which imposes a too severe restriction.

These considerations lead us to the question how a sequence of durative actions can be turned into a sequential macro to be used for temporal planning (in place of its constituent actions), so that the resulting plan guarantees the sequential applicability of the original actions but does not suppress other concurrent actions unnecessarily.
Adhering to these design objectives, the principal steps of our construction of macro-actions are outlined in Figure~\ref{fig:macro}:
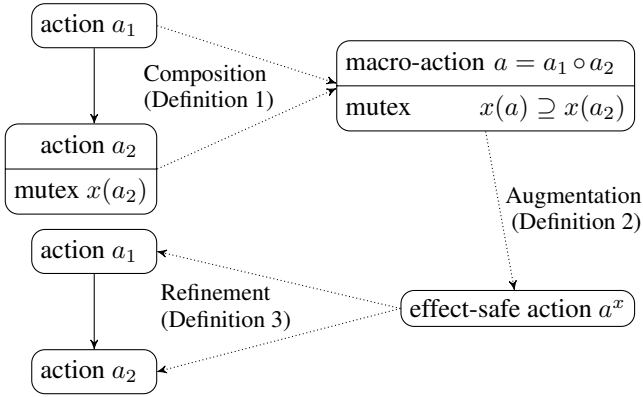
\begin{figure}[t]
\hspace*{-4pt}
\begin{tikzpicture}[x=1pt,y=1pt]
% \node[draw,rounded corners=5pt,rectangle split,rectangle split parts=2,rectangle split part align={right, right}] (a1) at (0,0) {\nodepart{one}action \phantom{$\aux($}$\act_1$\phantom{$)$}\nodepart{two}mutex $\auxa{\act_1}$};
\node[draw,rounded corners=5pt,rectangle split,rectangle split parts=2,rectangle split part align={right, right}%,below=of a1
] (a2) {\nodepart{one}action $\act_2$\phantom{$)$}\nodepart{two}mutex $\auxa{\act_2}$};
\node[draw,rounded corners=5pt,above=of a2.north east,anchor=south east] (a1) {action {\rlap{$a_1$}}\phantom{$a_2)$}};
%\draw[-latex] (a1) -- (a2) node[midway] (aa) {};
\draw[-stealth'] let \p1 = (a1.south), \p2 = (a2.north) in (\x1,\y1) -- (\x1,\y2) node[midway] (aa) {};
\node[draw,rounded corners=5pt,rectangle split,rectangle split parts=2,rectangle split part align={left, left},text width=106pt,right=88pt of aa] (a) {\nodepart{one}macro-action $\act=\comp{\act_1}{\act_2}$\phantom{$)$}\nodepart{two}mutex\hfill$\auxa{\act}\supseteq%\auxa{\act_1}\cup
\auxa{\act_2}$};
\draw[densely dotted,-stealth'] (a1.east) -- ([yshift=1pt]a.west) node[left,xshift=-14pt,yshift=-2pt] {\footnotesize\begin{tabular}{l}Composition\\(Definition \ref{def:macro})\end{tabular}};
\draw[densely dotted,-stealth'] (a2.east) -- ([yshift=-1pt]a.west);
\node[draw,rounded corners=5pt,below=60pt of a.south east,anchor=north east] (ax) {effect-safe action $\auxs{\act}$};
\draw[densely dotted,-stealth'] (a.south) -- (ax) node[midway,right,xshift=-1.5pt] {\footnotesize\begin{tabular}{@{}r@{}}Augmentation\\(Definition \ref{def:effect})\end{tabular}};
\node[above left=of ax] (xx) {\phantom{action $a_2$}};
\path let \p1 = (a2.east), \p2 = ($ (ax) !.5! (xx) $) in node[draw,rounded corners=5pt,anchor=east] at (\x1,\y2) (aa1) {action {\rlap{$a_1$}}\phantom{$a_2)$}};
\node[draw,rounded corners=5pt,below=of aa1] (aa2) {action $a_2$\phantom{$)$}};
\draw[-stealth'] (aa1) -- (aa2);
\draw[densely dotted,-stealth'] (ax.west) -- (aa1.east) node[below,xshift=26pt,yshift=-8pt] {\footnotesize\begin{tabular}{l}Refinement\\(Definition \ref{def:refine})\end{tabular}};
\draw[densely dotted,-stealth'] (ax.west) -- (aa2.east);
\end{tikzpicture}
\caption{Outline of the composition and refinement of macro-actions,
augmented with mutex atoms for sound temporal planning.\label{fig:macro}}
\end{figure}
\begin{enumerate}
\item
% To start with, 
Definition~\ref{def:macro} specifies how the sequential execution of two actions, $\act_1$ and $\act_2$, is mapped to a macro-action $\act=\comp{\act_1}{\act_2}$, where the internal structure incorporating the ending event for $\act_1$, the starting event for $\act_2$, as well as the invariants $\pred{\act_1}$ and $\pred{\act_2}$ requires particular attention.
In case some literal $\lit$ needs to be excluded from the effects of events taking place within the duration of the macro-action~$\act$,
we introduce a mutex atom $\auxv{\lit}$ in the set $\auxa{\act}$ associated with~$\act$, which extends the corresponding set % s $\auxa{\act_1}$ and 
$\auxa{\act_2}$ for $\act_2$. % its constituent actions.
This inductive accumulation of mutex atoms accommodates the right-associative chaining of macro-action composition steps,
starting from % $\auxa{\act_1}=
$\auxa{\act_2}=\emptyset$ for % original actions $\act_1$ and $\act_2$.
an ordinary action $\act_2$.
\item
After composing macro-actions~$\act$ and gathering their associated mutex atoms $\auxa{\act}$, Definition~\ref{def:effect} incorporates mutex atoms into the precondions and effects of effect-safe (macro-)actions $\auxs{\act}$.
The main idea is that the precondition of any event is augmented with $\auxv{\lit}$ for corresponding effects $\lit$ that must not apply during some macro-action~$\act$.
Such an effect-safe macro-action $\auxs{\act}$ in turn falsifies $\auxv{\lit}$ at the start and re-enables it at the end, thus suppressing undesired effects of events and also ruling out interferences with (other) macro-actions whose associated mutex atoms include $\auxv{\var}$ or $\auxv{\neg\var}$ for $\lit\in\{\var,\neg\var\}$.
The latter restriction on the concurrent applicability of macro-actions guarantees that mutex atoms are neither manipulated in uncontrolled ways nor that unfolding macro-actions into their constituent actions risks the release of undesired effects.
\item
Given a solution for a planning task built from effect-safe (macro-)actions,
Definition~\ref{def:refine} formalizes how a time-stamped macro-action $(\stamp,\act)$ with $\act=\comp{\act_1}{\act_2}$ is unfolded into the sequence of
$(\stamp_1,\act_1)$ and $(\stamp_2,\act_2)$ to obtain a refined plan.
The introduced time stamps $\stamp_1$ and $\stamp_2$ are chosen such that
$\stamp_1<\stamp$, $\stamp_1+\dura{\act_1}<\stamp_2$, and $\stamp_2+\dura{\act_2}<\stamp+\dura{\act}$, 
where no other starting or ending event takes place in-between
$\stamp_1$ and $\stamp$, $\stamp_1+\dura{\act_1}$ and $\stamp_2$, or
$\stamp_2+\dura{\act_2}$ and $\stamp+\dura{\act}$, respectively.
That is, the starting event for $\act_1$ (or ending event for $\act_2$)
replaces the starting event for $\act$ (or ending event for $\act$), and
the start of $\act_2$ directly succeeds the end of $\act_1$ in the refined plan.
The property that stepwise refinement preserves solutions is stated by Theorem~\ref{thm:refine}, and Corollary~\ref{cor:macro} concludes that all possible orders of refinement steps lead to a similar solution in terms of ordinary (time-stamped) actions.
\end{enumerate}

Our macro-action concept combines preconditions and
effects at the start and end of composed actions as well as their invariants in a fine-grained way, based on the idea of incorporating
invariants if they do not spoil the applicability of a macro-action,
or to gather mutex atoms on literals otherwise. % that must be subject to mutex locks otherwise.
\begin{definition}\label{def:macro}
For two actions $\act_1$ and $\act_2$, let
$\pre^1 = \pred{\act_1}\cup\pree{\act_1}$,
$\pre^2 = \pred{\act_2}\cup\pree{\act_2}$,
$\efp^1=\effsp{\act_1}\cup\effep{\act_1}$,
$\efm^1=\effsm{\act_1}\cup\effem{\act_1}$,
$\efp^{12}=\effep{\act_1}\cup\effsp{\act_2}$,
$\efm^{12}=\effem{\act_1}\cup\effsm{\act_2}$, and
$\efm^2=\effsm{\act_2}\cup\effem{\act_2}$
abbreviate particular unions of preconditions or
positive and negative effects of $\act_1$ and $\act_2$.
We define the duration, preconditions, and effects of the
right-associative composition of $\act_1$ and $\act_2$
into the \emph{macro-action} $\act=\comp{\act_1}{\act_2}$ by:
\begin{enumerate}
\item $\dura{\act}=\dura{\act_1}+\dura{\act_2}$
      is the sum of the durations of $\act_1$ and $\act_2$;
\item the preconditions $\pres{\act}$
      at the start of $\act$ are the union of the following sets of
      atoms:
      \begin{enumerate}
      \item the original preconditions $\pres{\act_1}$
            at the start of $\act_1$,
      \item the set $\pre^1\cap\efm^{12}\setminus\effsp{\act_1}$
            of preconditions during or at the end of $\act_1$
            that get falsified by effects $\efm^{12}$
            and are not readily enabled by the start effects $\effsp{\act_1}$, and
      \item the set $\pres{\act_2}\cap\effsm{\act_2}\setminus\efp^1$
            of preconditions at the start of $\act_2$
            that get falsified by the start effects $\effsm{\act_2}$
            and are not enabled by effects $\efp^1$ of $\act_1$; % , and
%      \item the set $\auxa{\act}$ of auxiliary atoms introduced above;
      \end{enumerate}
\item the invariants $\pred{\act}$ during the macro-action $\act$ are the
      union of the following sets of
      atoms:
      \begin{enumerate}
      \item the set $\pre^1\setminus(\efm^{12}\setminus\effsm{\act_1})$
            of preconditions during or at the end of $\act_1$
            that do not get falsified by effects $\efm^{12}$
            (and also not by $\effsm{\act_1}$, as
             $\effsm{\act}\cap\pred{\act}\neq\emptyset$
             leaves
             $\act=\comp{\act_1}{\act_2}$
             undefined otherwise),
      \item the set $\pres{\act_2}\setminus(\effep{\act_1}\cup(\effsm{\act_2}\setminus\efm^1))$
            of preconditions at the start of $\act_2$
            that are neither readily enabled by the end effects $\effep{\act_1}$ nor falsified by the start effects $\effsm{\act_2}$ (but not by $\efm^1$ as well, as such a situation yields $\effsm{\act}\cap\pred{\act}\neq\emptyset$, in which case $\act=\comp{\act_1}{\act_2}$ is taken to be undefined), and
      \item the set $\pred{\act_2}\setminus\efp^{12}$
            of invariants during $\act_2$ that are not enabled
            by effects $\efp^{12}$;
      \end{enumerate}
\item $\pree{\act}=\pree{\act_2}\setminus\efp^{12}$ 
      contains original preconditions $\pree{\act_2}$ at the end of $\act_2$
      that are not enabled by effects $\efp^{12}$;
\item the start effects $\effs{\act}$ of $\act$
      are the union of the following sets of
      literals:
      \begin{enumerate}
      \item the set $\effs{\act_1}\setminus\efm^{12}$
            of original start effects of $\act_1$ that do not get
            falsified by effects $\efm^{12}$, and
      \item the set $\{\neg\var \mid \var\in\efm^{12}\}$
            of negative literals anticipating the falsification of atoms by effects $\efm^{12}$; % , and
%      \item the set $\{\neg\aux \mid \aux\in\auxa{\act}\}$
%            of negative literals over the auxiliary atoms $\auxa{\act}$ introduced above;
      \end{enumerate}
\item the end effects $\effe{\act}$ of $\act$
      are the union of the following sets of
      literals:
      \begin{enumerate}
      \item the original end effects $\effe{\act_2}$ of $\act_2$, and
      \item the set $\efp^{12}\setminus\efm^2$
            of atoms that are enabled by effects $\efp^{12}$
            and do not get falsified by effects $\efm^2$; % , and
 %     \item the set $\auxa{\act}$ of auxiliary atoms;
      \end{enumerate}
\end{enumerate}
provided that
$\effsm{\act}\cap\pred{\act}=\emptyset$ % ,
% $\pree{\act_1}\cap\effsm{\act_1}
% =\emptyset$,
and
$\pree{\act_2}\cap\efm^{12}\setminus\effsp{\act_2}
=\emptyset$ hold,
while $\act=\comp{\act_1}{\act_2}$ is undefined otherwise.

If the macro-action $\act$ is defined, by
\begin{align*}
%\nonumber
\auxa{\act}=
            \{\auxv{\neg\var}\mid{}&
                \begin{array}[t]{@{}r@{}l@{}}
                \var\in\big(&
                (\efp^{12}\setminus(\effep{\act_2}\cup\efm^2)) \\ {}\cup{} &
                (\pre^1\cap\efm^{12})\\ {}\cup{} &
                (\pres{\act_2}\cap\effsm{\act_2}\setminus\effep{\act_1})\\ {}\cup{} &
                (\pre^2\cap\efp^{12})
                \big)\setminus\pred{\act}\}
                \end{array}
\\%\label{eq:aux}
{}\cup
                \{\auxv{\var}\mid{}&\var\in\efm^{12}\}
                \cup
                % \auxa{\act_1}\cup
                \auxa{\act_2}
\text{,}
\end{align*}
we denote a set of \emph{mutex atoms} $\auxv{\var}$ and $\auxv{\neg\var}$
associated with (the complements of) literals $\var$ and $\neg \var$
% over the original atoms $\var$ 
occurring in effects and preconditions
of $\act_1$ and $\act_2$,
where we take $\auxa{\act_2}=\emptyset$ % $\auxa{\act_1}=\emptyset$ (or $\auxa{\act_2}=\emptyset$)
as base case if $\act_2$ % $\act_1$ (or $\act_2$) 
is not itself a macro-action defined as the composition of two actions.\qed
\end{definition}

\begin{figure}[t]
\centering
\includegraphics[width=0.43\textwidth]{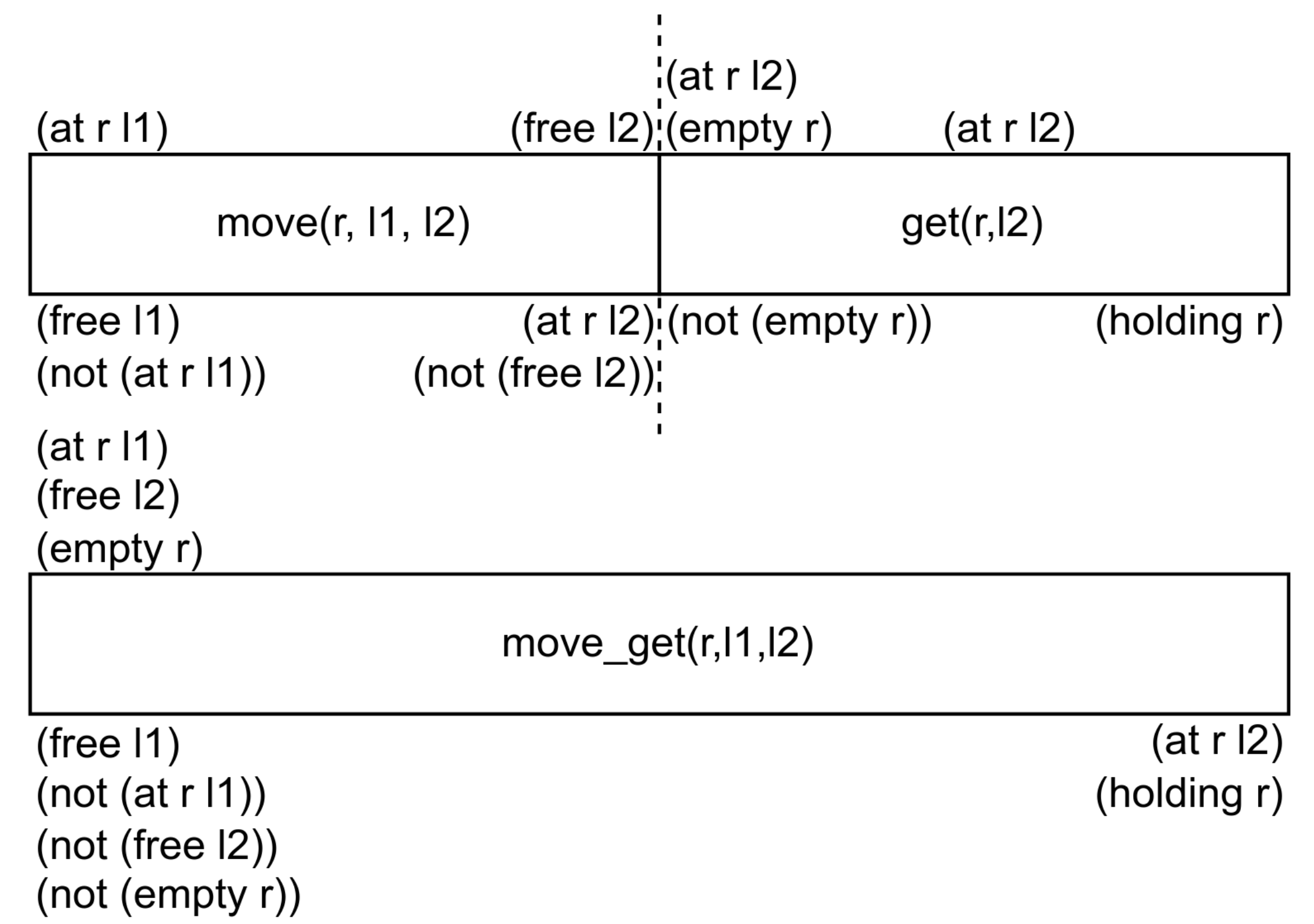} 
\caption{Description of two actions and the resulting macro. Literals below each action represent effects, while the literals above them provide preconditions (at the start, end, or during an action, as indicated by their positions). The mutex atoms of the macro are omitted for better readability.}
\label{macro}
\end{figure}

\begin{example}\label{ex:compose}
To illustrate the composition of macro-actions, Figure~\ref{macro}
visualizes how two actions from a simple temporal domain are composed into a macro-action. The actions at the top involve an agent $r$ capable of moving from a location $l1$ to $l2$ for picking up an object at location $l2$.
(The syntax used for literals is inspired by PDDL, where $\mathit{not}$
represents the logical connective $\neg$.)
First observe that the macro-action displayed at the bottom of Figure~\ref{macro} pulls the delete effects applied at the end of the $\mathit{move}$ or at the start of the $\mathit{get}$ action, i.e., $(\mathit{not}\: (\mathit{free}\: l2))$ and $(\mathit{not}\: (\mathit{empty}\: r))$, together with the original start effects of the $\mathit{move}$ action.
The positive end effect $(\mathit{at}\ r\ l2)$ of the $\mathit{move}$ action, however, joins $(\mathit{holding}\ r)$
at the end of the composed macro-action.

Preconditions at the start of the macro-action include the original
$(\mathit{at}\: r\: l1)$ atom from $\mathit{move}$ together with
$(\mathit{free}\: l2)$ and $(\mathit{empty}\: r)$ required at the
end of $\mathit{move}$ or at the start of $\mathit{get}$, respectively.
The reason for not turning the latter two atoms into invariants required
throughout the macro-action is that their negative literals occur as new
start effects, so that invariants would render the macro-action inapplicable.
Moreover,
the precondition and invariant $(\mathit{at}\: r\: l2)$ of $\mathit{get}$ is not taken as
a precondition or invariant of the macro-action since it is enabled by
the end effect of $\mathit{move}$, which is now postponed to the end of
the macro-action.
Hence, it would be overcautious to insist on the truth
of $(\mathit{at}\: r\: l2)$ at the start or during the entire macro-action.
In fact, considering that any other actions in the domain will hardly admit
$(\mathit{at}\: r\: l1)$, which is a precondition at the start, and
$(\mathit{at}\ r\ l2)$ to hold simultaneously, turning the latter into a precondition or invariant would most likely yield an (unnoticed)
inapplicable macro-action.\qed
\end{example}

In general, the idea of Definition~\ref{def:macro} is to forward delete effects $\effem{\act_1}\cup\effsm{\act_2}$
($(\mathit{not}\: (\mathit{free}\: l2))$ and $(\mathit{not}\: (\mathit{empty}\: r))$ in Example~\ref{ex:compose})
from the ending event for $\act_1$ or the starting event for $\act_2$ to the start of the composed macro-action $\act = \comp{\act_1}{\act_2}$.
In this way, the preconditions of other actions can right from the start not build on atoms getting falsified in the course of the macro-action.
Similarly, the add effects $\effep{\act_1}\cup\effsp{\act_2}$ ($(\mathit{at} \: r\: l2)$ in Example~\ref{ex:compose}),
which may enable preconditions of other actions, are postponed to
the end of the macro-action (unless they get canceled by subsequently
occurring delete effects $\effsm{\act_2}$ and $\effem{\act_2}$).
Taken together, the early application of delete effects and postponement of add effects, which would
originally arise at some point during the macro-action, prevent that other actions building on the volatile
atoms are applied.

Concerning the preconditions of the macro-action~$\act$,
$\pree{\act}$ at the end consists of the atoms in $\pree{\act_2}$ that are not
enabled by the add effects $\effep{\act_1}\cup\effsp{\act_2}$ during
the macro-action.
Note that atoms in $\effsp{\act_2}$ may also belong to the delete effects $\effem{\act_1}$, in which case they are included in $\effsm{\act}$ at the start of~$\act$ and listing them among the preconditions $\pree{\act}$ at the end would render~$\act$
inapplicable.
For the same reason, atoms of the invariant $\pred{\act_2}$ that
get enabled by $\effep{\act_1}\cup\effsp{\act_2}$
during $\act$ are not required as invariants in $\pred{\act}$ for not (unnecessarily) compromising the applicability of~$\act$.
Atoms of the invariant $\pred{\act_1}$ as well as the
preconditions $\pree{\act_1}$ and $\pres{\act_2}$ during $\act$
are taken as invariants in $\pred{\act}$ only if they are not falsified by subsequent delete effects in $\effem{\act_1}$ or
$\effsm{\act_2}$.
Otherwise, such atoms ($(\mathit{free}\: l2)$ and $(\mathit{empty}\: r)$ in Example~\ref{ex:compose}) augment the original preconditions $\pres{\act_1}$ at the start of $\act_1$ in $\pres{\act}$   (unless they get readily enabled by add effects $\effsp{\act_1}$ and $\effep{\act_1}$, as with $(\mathit{at} \: r\: l2)$ in Example~\ref{ex:compose}).
While the composition of $\act_1$ and $\act_2$ into $\act=\comp{\act_1}{\act_2}$ aims at restricting the preconditions
$\pres{\act}$, $\pred{\act}$, and $\pree{\act}$ to necessary parts,
it can happen that $\act_1$ and $\act_2$ are incompatible in
the sense that delete effects undo required preconditions, and
checking that
$\effsm{\act}\cap\pred{\act}=\emptyset$ 
as well as
$\pree{\act_2}\cap(\effem{\act_1}\cup\effsm{\act_2})\setminus\effsp{\act_2}
=\emptyset$
in Definition~\ref{def:macro} 
excludes the composition of incompatible actions.

Although the specific delete effects in Example~\ref{ex:compose} do
not permit taking atoms as invariants of the composed macro-action
$\act=\comp{\act_1}{\act_2}$,
it would be the first choice for, e.g., $(\mathit{empty}\: r)$ from
$\pres{\act_2}$ if it were not also included in $\effsm{\act_2}$.
If this choice cannot be made for an atom $\var$ of interest, as in Example~\ref{ex:compose}, a mutex atom $\auxv{\neg\var}$ is
collected in $\auxa{\act}$ to express that any delete effects on
$\var$ need to be rejected as long as the macro-action $\act$ is
in progress.
The respective cases in Definition~\ref{def:macro} cover all atoms
from the preconditions $\pred{\act_1}$, $\pree{\act_1}$, and
$\pres{\act_2}$ subject to subsequent delete effects,
atoms from $\pred{\act_2}$ and $\pree{\act_2}$ getting enabled
in the course of the macro-action~$\act$, as well as postponed
effects from $\effep{\act_1}$ and $\effsp{\act_2}$ that are
not to be canceled before the ending event for $\act$
(unless any of these atoms belongs to the invariants of $\act$
 by Definition~\ref{def:macro}).
Additional mutex atoms of the form $\auxv{\var}$
are included in $\auxa{\act}$ for delete effects
$\effem{\act_1}\cup\effsm{\act_2}$ occurring during~$\act$.
They signal that add effects on $\var$ must be rejected to
prevent concurrent actions building on atoms that get falsified
during~$\act$.
The mutex atoms associated with macro-actions are then used to
model mutex locks.

\begin{definition}\label{def:effect}
For a planning task $\task = (\vars,\acts,\init,\goal)$,
let $\auxu=\bigcup_{\act\in\acts}\auxa{\act}$ such that
$\auxu\cap\vars=\emptyset$ be the set of mutex atoms for
(macro-)actions $\act\in\acts$.
We define the \emph{effect-safe task} 
$\auxs{\task}=(\auxs{\vars},\linebreak[1]\auxs{\acts},\linebreak[1]\auxs{\init},\linebreak[1]\goal)$
for $\task$ by:
\begin{enumerate}
\item $\auxs{\vars}=\vars\cup\auxu$;
\item 
$\auxs{\acts}=
 \{\auxs{\act}\mid \act\in\acts\}$ with
  \begin{enumerate}
  \item $\dura{\auxs{\act}}=\dura{\act}$,
  \item $\pres{\auxs{\act}}=\pres{\act}\cup\auxa{\act}$
  
  \hfill${}\cup
                                      \{\auxv{\var}\in\auxu\mid\auxv{\neg\var}\in\auxa{\act}\}\cup
                                      \{\auxv{\lit}\in\auxu 
                                        \mid 
                                        \lit\in\effs{\act}\}$,
  \item $\pred{\auxs{\act}}=\pred{\act}$,
  \item $\pree{\auxs{\act}}=\pree{\act}\cup
                                      \{\auxv{\lit}\in\auxu
                                        \setminus\auxa{\act}
                                        \mid 
                                        \lit\in\effe{\act}\}$,
  \item $\effs{\auxs{\act}}=\effs{\act}\cup\{\neg\auxv{\lit} \mid\auxv{\lit}\in\auxa{\act}\}$, and
  \item $\effe{\auxs{\act}}=\effe{\act}\cup\auxa{\act}$;
  \end{enumerate}
  \item $\auxs{\init}=\init\cup\auxu$.
\end{enumerate}
If $\auxs{\plan}$ is a solution for $\auxs{\task}$,
we call $\plan=\{(\stamp,\act) \mid (\stamp,\auxs{\act})\in\auxs{\plan}\}$ the \emph{base plan} of $\auxs{\plan}$.
\qed
\end{definition}

\begin{example}\label{ex:mutex}
Continuing Example~\ref{ex:compose},
the mutex atoms
(omitted in Figure \ref{macro} for better readability)
associated with the composed $\mathit{move}$
and $\mathit{get}$ macro-action
are
$\auxv{(\mathit{not} \: (\mathit{free} \: l2))}$,
$\auxv{(\mathit{not} \: (\mathit{empty} \: r))}$,
$\auxv{(\mathit{not} \: (\mathit{at} \: r \: l2))}$,
$\auxv{(\mathit{free} \: l2)}$, and
$\auxv{(\mathit{empty} \: r)}$.
In the $\mathit{not}$ cases,
they result from preconditions at the end of $\mathit{move}$ and
at the start of or during $\mathit{get}$ that are canceled by subsequent 
delete effects or enabled during the macro-action, respectively.
On the other hand,
$\auxv{(\mathit{free} \: l2)}$ and
$\auxv{(\mathit{empty} \: r)}$ stem from the delete effects at the
end of $\mathit{move}$ and
at the start of $\mathit{get}$.

The effect-safe version of the macro-action includes the above mutex atoms
as well as
$\auxv{(\mathit{at} \: r \: l2)}$,
$\auxv{(\mathit{free} \: l1)}$, and
$\auxv{(\mathit{not} \: (\mathit{at} \: r \: l1))}$ as
additional preconditions at the start
(provided that the latter show up among the mutex atoms
 associated with other macro-actions composed for the domain)
to rule out any interferences with other macro-actions on start effects
or the original atoms~$\var$ occurring as $\auxv{\var}$ or $\auxv{\neg\var}$ among mutex atoms.
When the macro-action composed from $\mathit{move}$
and $\mathit{get}$ is applied, its associated mutex atoms
are set to false at the start in order to reject undesired effects of other actions,
i.e., effects falsifying some required precondition 
or enabling an atom falsified during the macro-action (too early).
These mutex locks get released again at the end of the macro-action,
where $\auxv{(\mathit{at} \: r \: l2)}$ and
$\auxv{(\mathit{holding} \: r)}$ for the end effects
$(\mathit{at} \: r \: l2)$ and
$(\mathit{holding} \: r)$ constitute preconditions (in case
any other macro-actions have them as associated mutex atoms).
Importantly, if either of these atoms were among the mutex atoms
associated with the macro-action itself,
it would not be taken as a precondition for the ending event;
e.g., if $(\mathit{not} \: (\mathit{holding} \: r))$ were an effect at the
start of the $\mathit{get}$ action,
$\auxv{(\mathit{holding} \: r)}$
would be included in the mutex atoms, so that
$\auxv{(\mathit{holding} \: r)}$ is certainly false until
the end of the macro-action due to the modeled mutex lock.
\qed
\end{example}

When several actions are to be chained into a composite macro-action,
we suppose that Definition~\ref{def:macro} is applied from right to left,
i.e., that the macro-action is successively composed from the end.
The reason is that the internal structure of a macro is hidden 
to the outside, and further precaution would be needed if a macro-action
were extended on to later actions in a sequence.
To see this, consider three actions $\act_1$, $\act_2$, and $\act_3$ such
that an atom $\var$ belongs to the delete effects $\effem{\act_1}$ as well as the
preconditions $\pree{\act_3}$ at the end of~$\act_3$.
Then, checking whether
$\pree{\act_3}\cap\effem{\act_1}\setminus\effsp{\comp{\act_2}{\act_3}}
=\emptyset$
according to Definition~\ref{def:macro}
discards the macro-action composition $\comp{\act_1}{(\comp{\act_2}{\act_3})}$,
while $\var$ is forwarded to the delete effects $\effsm{\act}$
at the start of $\act=\comp{\act_1}{\act_2}$ otherwise.
In this case, the applicability conditions in Definition~\ref{def:macro} are too
simplistic to detect that $\var$ cannot hold at the end of $\comp{\act}{\act_3}$, considering that
effects enabling $\var$ are rejected by falsifying $\auxv{\var}$ at the 
start of $\auxs{\act}$. % $\act$ or $\comp{\act}{\act_3}$, respectively.

Since the preconditions and effects on the mutex atoms
$\auxa{\act}$ for macro-actions $\act$ put additional restrictions
on their concurrent applicability, any solution for an effect-safe task
is also guaranteed to yield a solution for the original planning task,
while the converse does not hold in general.

\begin{proposition}\label{prop:macro}
Let $\task = (\vars,\acts,\init,\goal)$ be a planning task, and
$\auxs{\plan}$ be a solution % of size $n-1$
for the effect-safe
task $\auxs{\task}=(\auxs{\vars},\linebreak[1]\auxs{\acts},\linebreak[1]\auxs{\init},\linebreak[1]\goal)$.
Then, 
% $\plan=\{(\stamp,\act) \mid (\stamp,\auxs{\act})\in\auxs{\plan}\}$
the base plan $\plan$ of $\auxs{\plan}$
is a solution for $\task$.\qed
\end{proposition}

\begin{proof}
The preservation of plan consistency when turning
$\auxs{\plan}$ into $\plan=\{(\stamp,\act) \mid (\stamp,\auxs{\act})\in\auxs{\plan}\}$ and the correspondence of
time-stamped states 
% in the respective sequences $\auxs{\seq}$ and $\seq$ 
on the atoms of $\vars$
directly follow from Definition~\ref{def:effect}, given that 
$\auxs{\init}\cap \vars=\init$ and
$\pres{\auxs{\act}}\cap \vars=\pres{\act}$,
$\pred{\auxs{\act}}=\pred{\act}$,
$\pree{\auxs{\act}}\cap \vars=\pree{\act}$,
$\effs{\auxs{\act}}\cap \vars=\effs{\act}$, and
$\effe{\auxs{\act}}\cap \vars=\effe{\act}$
for each action $\auxs{\act}\in\nolinebreak\auxs{\acts}$.
\end{proof}

\begin{example}\label{ex:effect}
Consider a macro-action $\act_0=\comp{\act_1}{\act_2}$
such that $\var_1\in\effsm{\act_1}\cap\pree{\act_2}$,
$\var_2\in\effem{\act_1}$, and
$\var_3\in\effep{\act_2}$.
Then, the mutex atom $\auxv{\var_2}$ is included in
$\effsm{\auxs{\act_0}}$ to reject volatile add effects on
$\var_2$ during the effect-safe version of $\act_0$.
Hence, the starting (or ending) event for any action $\act$ with
$\{\var_1,\var_2\}\subseteq\effsp{\act}$ and
$\auxv{\var_2}\in\pres{\auxs{\act}}$
(or
$\{\var_1,\var_2\}\subseteq\effep{\act}$ and
$\auxv{\var_2}\in\pree{\auxs{\act}}$)
cannot take place in-between the start and end of $\auxs{\act_0}$.
As a consequence, there is no solution for the effect-safe task
$\auxs{\task}=(\{\var_1,\var_2,\var_3\}\cup\{\auxv{\var_2}\},\linebreak[1]
               \{\auxs{\act_0},\auxs{\act}\},\linebreak[1]\{\auxv{\var_2}\},\linebreak[1]\{\var_3\})$,
while
$\task=(\{\var_1,\var_2,\var_3\},\linebreak[1]
               \{\act_0,\act\},\linebreak[1]\emptyset,\linebreak[1]\{\var_3\})$
admits solutions such that $\act$ starts (or ends) within the
duration of $\act_0$.\qed
\end{example}

Example~\ref{ex:effect} shows that the replacement of ordinary actions
by macro-actions, as performed in Section~\ref{evaluation} to increase
the scalability of temporal planning,
needs to be done with care, as it depends on the domain at hand whether
a stricter effect-safe task suppresses interactions between concurrent actions
that are crucial for the satisfiability of a planning task.
However, if a solution for an effect-safe task exists,
it can be refined by unfolding the contained macro-actions $\act=\comp{\act_1}{\act_2}$ into sequential applications of $\act_1$ and $\act_2$ such that
the ending event for $\act_1$ directly precedes the starting event for $\act_2$,
thus reproducing the sequence of events anticipated in the composition of $\act$.

\begin{definition}\label{def:refine}
For a solution $\auxs{\plan}$ % of size $n$ 
for an effect-safe task
$\auxs{\task}
 =(\auxs{\vars},\linebreak[1]\auxs{\acts},\linebreak[1]\auxs{\init},\linebreak[1]\goal)$,
we define the base plan $\plan$ of $\auxs{\plan}$
as a \emph{refined plan} of $\auxs{\plan}$ for the planning task
$\task=(\vars,\{\act\mid (\stamp,\act)\in\plan\},\init,\goal)$.
% with the \emph{refined time stamps}
% $\refine{i}=\{\stampi{i}\}$
% for $0\leq i\leq 2n$.
%
If $\auxz{\plan}$ with $(\stamp_0,\act_0)\in\auxz{\plan}$
for a macro-action $\act_0=\comp{\act_1}{\act_2}$ is a refined plan of $\auxs{\plan}$ for the planning task
$\auxz{\task}=(\vars,\{\act\mid (\stamp,\act)\in\auxz{\plan}\},\init,\goal)$,
%with the refined time stamps $\refinez{i}$ for $0\leq i\leq 2n$,
then we define
$\plan=\{(\stamp_1,\act_1),(\stamp_2,\act_2)\}
       \cup\auxz{\plan}\setminus\{(\stamp_0,\act_0)\}$ such that\pagebreak[1]
\begin{equation*}%\label{eq:min}
\begin{array}[b]{@{}r@{}l@{}}
\deltax{\auxz{\plan}}{(\stamp_0,\act_0)}=
\min\big(
{\textstyle\bigcup_{0\leq i < 2|\auxz{\plan}|}}
(
\{\stamp_0-\stampi{i}\mid
% % \stamp\in\bigcup_{0\leq i\leq 2n}\refinez{i},\linebreak[1]
% (\stamp,\act)\in\auxz{\plan},
% 0\leq i < 2|\auxz{\plan}|,
\stampi{i}<\stamp_0\}
\\{}\cup
        \{\stamp_0+\dura{\act_1}-\stampi{i}\mid%\stamp\in\bigcup_{0\leq i\leq 2n}\refinez{i},\linebreak[1]
% (\stamp,\act)\in\auxz{\plan},
% 0\leq i < 2|\auxz{\plan}|,
        \stampi{i}<\stamp_0+\dura{\act_1}\}
\\{}\cup
        \{\stamp_0+\dura{\act_0}-\stampi{i}\mid%\stamp\in\nolinebreak\bigcup_{0\leq i\leq 2n}\refinez{i},\linebreak[1]
% (\stamp,\act)\in\auxz{\plan},
% 0\leq i < 2|\auxz{\plan}|,
        \stampi{i}<\stamp_0+\dura{\act_0}\}&)\big)
\text{,}
\end{array}
\end{equation*}
$\stamp_0-\deltax{\auxz{\plan}}{(\stamp_0,\act_0)}<\stamp_1<\stamp_0$, and
$\stamp_1+\dura{\act_1}<\stamp_2<\stamp_0+\dura{\act_1}$
as a \emph{refined plan} of $\auxs{\plan}$ for the planning task
$\task=(\vars,\{\act\mid (\stamp,\act)\in\plan\},\init,\goal)$.\qed
\end{definition}

Our main result on refined plans, obtained by unfolding macro-actions
into sequences of their constituents, is that they are guaranteed to
be solutions.
Hence, an effect-safe task allows for performing temporal planning at the level
of macro-actions, and then the original actions can be put back without revising and possibly discarding the resulting plan.

\begin{theorem}\label{thm:refine}
Let $\auxs{\plan}$ be a solution for an effect-safe task
$\auxs{\task}=(\auxs{\vars},\linebreak[1]\auxs{\acts},\linebreak[1]\auxs{\init},\linebreak[1]\goal)$.
Then, any refined plan $\plan$ of $\auxs{\plan}$ for
$\task=(\vars,\linebreak[1]%\acts\cup
\{\act\mid (\stamp,\act)\in\plan\},\linebreak[1]\init,\linebreak[1]\goal)$
is a solution for $\task$.\qed
\end{theorem}

\begin{proof}
For any refined plan $\plan$ of $\auxs{\plan}$ for 
the planning task $\task=(\vars,\linebreak[1]%\acts\cup
\{\act\mid (\stamp,\act)\in\plan\},\linebreak[1]\init,\linebreak[1]\goal)$,
an atom $\var\in\vars$, and % two 
time stamps
$\{\stamp_1,\stamp_2\}\subseteq\{\stampi{i} \mid 1\leq i\leq 2|\plan|\}$
 % \{\stamp,\stamp+\dura{\act} \mid
 %  (\stamp,\act)\in\plan\}$
with $% 0<
\stamp_1\leq\stamp_2$, let 
\begin{equation*}
\begin{array}{@{}r@{}l@{}}
      %\nonumber
      \refinex{\plan}{\var}{(\stamp_1,\stamp_2)}=
       {}&\min\big(\{\stampi{2|\plan|}\} %\max(\{\stamp+\dura{\act} \mid (\stamp,\act)\in \plan\})
       \\{}\cup{}&
            \big\{\stamp+\dura{\act} \mid 
             (\stamp,\act)\in\plan,
              \stamp+\dura{\act}<\stamp_1,
              {}\\&
              \phantom{\{}\var\in\effep{\act}\cap{\textstyle\bigcap}_{\ind{\stamp+\dura{\act}}\leq i<\ind{\stamp_2}}\statei{i},
              {}\\&
            \begin{array}[t]{@{}r@{}}
              \phantom{\{}\{(\stamp_0,\act_0)\in\plan \mid
                \auxv{\var}\in\auxa{\act_0},
                \stamp_0<\stamp_2,
              {}\\
                \stamp+\dura{\act}<\stamp_0+\dura{\act_0}\}=\emptyset\big\}
            \end{array}
      \\
            {}\cup{}&
            \big\{\stamp \mid
             (\stamp,\act)\in\plan,
              \stamp<\stamp_1,
              {}\\&
              \phantom{\{}\var\in\effsp{\act}\cap{\textstyle\bigcap}_{\ind{\stamp}\leq i<\ind{\stamp_2}}\statei{i},
              {}\\&
            \begin{array}[t]{@{}r@{}}
              \phantom{\{}\{(\stamp_0,\act_0)\in\plan \mid
                \auxv{\var}\in\auxa{\act_0},
                \stamp_0<\stamp_2,
              {}\\
                \stamp<\stamp_0+\dura{\act_0}\}=\emptyset\big\}
            \end{array}
            \\
            {}\cup{}&
            \big\{0 \mid
              \var\in{\textstyle\bigcap}_{0\leq i<\ind{\stamp_2}}\statei{i},
              {}\\&\phantom{\{}
              \{(\stamp_0,\act_0)\in\plan \mid
                \auxv{\var}\in\auxa{\act_0},
                \stamp_0<\stamp_2\}=\emptyset\big\}\big)
\end{array}
\end{equation*}
denote the least time stamp~$\stamp$ smaller than $\stamp_1$
such that $\var$ remains true from $\statei{\ind{\stamp}}$ until
$\statei{\ind{\stamp_2}-1}$ and no macro-action $\act_0$ with
$\auxv{\var}\in\auxa{\act_0}$ is in progress at any of these states
(or the maximal time stamp at which some time-stamped action
in $\plan$ ends in case no such time stamp~$\stamp$ exists).

By Proposition~\ref{prop:macro},
we have that the base plan $\plan$ of $\auxs{\plan}$ is
a solution for the planning task
$\task=(\vars,\linebreak[1]\acts,\linebreak[1]\init,\linebreak[1]\goal)$.
Moreover, for any time-stamped action $(\stamp,\act)\in\plan$,
the condition \refer{5}{b} in Definition~\ref{def:macro} as well as
\refer{2}{b}, \refer{2}{d}, and \refer{2}{e} in Definition~\ref{def:effect} yield that
\begin{enumerate}
\item
$\refinex{\plan}{\var}{(\stamp,\stamp)}<\stamp$ for
each $\var\in\pres{\act}$,
\item
$\refinex{\plan}{\var}{(\stampi{\ind{\stamp}+1},\linebreak[1]\stamp+\dura{\act})}\leq\stamp$ for
each $\var\in\pred{\act}$, and
\item
$\refinex{\plan}{\var}{(\stamp+\dura{\act},\linebreak[1]\stamp+\dura{\act})}<\stamp+\dura{\act}$ for
each $\var\in\pree{\act}$.
\end{enumerate}

Let $\auxz{\plan}$ be a refined plan of $\auxs{\plan}$
for the planning task
$\auxz{\task} = 
 (\vars,\{\act\mid (\stamp,\act)\in\auxz{\plan}\},\init,\goal)$ such that
$\auxz{\plan}$ is a solution for $\auxz{\task}$
with $(\stamp_0,\act_0)\in\auxz{\plan}$
for a macro-action $\act_0=\comp{\act_1}{\act_2}$.
Assume that,
for any time-stamped action $(\stamp,\act)\in\auxz{\plan}$,
we have that
\begin{enumerate}
\item
$\refinex{\auxz{\plan}}{\var}{(\stamp,\stamp)}<\stamp$ for
each $\var\in\pres{\act}$,
\item
$\refinex{\auxz{\plan}}{\var}{(\stampi{\ind{\stamp}+1},\linebreak[1]\stamp+\dura{\act})}\leq\stamp$ for
each $\var\in\pred{\act}$, and
\item
$\refinex{\auxz{\plan}}{\var}{(\stamp+\dura{\act},\linebreak[1]\stamp+\dura{\act})}<\stamp+\dura{\act}$ for
each $\var\in\pree{\act}$.
\end{enumerate}
Pick some time stamps $\stamp_1$ and $\stamp_2$ such that
$\stamp_0-\deltax{\auxz{\plan}}{(\stamp_0,\act_0)}<\stamp_1<\stamp_0$ and
$\stamp_1+\dura{\act_1}<\stamp_2<\stamp_0+\dura{\act_1}$,
and consider the refined plan
$\plan=\{(\stamp_1,\act_1),(\stamp_2,\act_2)\}
       \cup\auxz{\plan}\setminus\{(\stamp_0,\act_0)\}$
of $\auxs{\plan}$ for the planning task
$\task=(\vars,\{\act\mid (\stamp,\act)\in\plan\},\init,\goal)$.

Then, for any time-stamped action
$(\stamp,\act)\in\auxz{\plan}\setminus\{(\stamp_0,\act_0)\}$,
% $\auxa{\act_1}\cup\auxa{\act_2}\subseteq\auxa{\act_0}$ 
$\auxa{\act_0}\supseteq\auxa{\act_2}$ by Definition~\ref{def:macro}
implies that
\begin{enumerate}
\item
$\refinex{\plan}{\var}{(\stamp,\stamp)}\leq
 \refinex{\auxz{\plan}}{\var}{(\stamp,\stamp)}<\stamp$ for
each $\var\in\pres{\act}$,
\item
$\refinex{\plan}{\var}{(\stampi{\ind{\stamp}+1},\linebreak[1]\stamp+\dura{\act})}\leq
 \refinex{\auxz{\plan}}{\var}{(\stampi{\ind{\stamp}+1},\linebreak[1]\stamp+\dura{\act})}\leq\stamp$ for
each $\var\in\pred{\act}$, and
\item
$\refinex{\plan}{\var}{(\stamp+\dura{\act},\linebreak[1]\stamp+\dura{\act})}\leq
 \refinex{\auxz{\plan}}{\var}{(\stamp+\nolinebreak\dura{\act},\linebreak[1]\stamp+\dura{\act})}<\stamp+\dura{\act}$ for
each $\var\in\pree{\act}$.
\end{enumerate}

Moreover,
the condition \refer{2}{a} in Definition~\ref{def:macro}
guarantees that
\begin{equation}\label{eq:proof:1}
\refinex{\plan}{\var}{(\stamp_1,\stamp_1)}=
 \refinex{\auxz{\plan}}{\var}{(\stamp_0,\stamp_0)}<\stamp_1
\end{equation}
for
each $\var\in\pres{\act_1}$.
By the conditions \refer{3}{a} and \refer{3}{b}
in Definition~\ref{def:macro},
we have that
\begin{equation*}
\begin{array}{@{}r@{}l@{}}
\big(&\pred{\act_1}\cup\pree{\act_1}\cup {}
\\  (&\pres{\act_2}\setminus\effep{\act_1})\big)\setminus \pred{\act_0}
\subseteq {}
\\
\big(&\pres{\act_0}\cup\effsp{\act_1}\big) \cap % {}
% \\
 \{\var\in\nolinebreak\vars \mid \auxv{\neg\var}\in\nolinebreak\auxa{\act_0}\}\text{,} % with
\end{array}
\end{equation*}
so that
\begin{equation}\label{eq:proof:2}
\begin{array}{@{}r@{}l@{}l@{}}
 \refinex{\plan}{\var}{(\stampi{\ind{\stamp_1}+1},\linebreak[1]\stamp_1+\dura{\act_1})}&
 {} \in {}
 \\
 \{
 \refinex{\auxz{\plan}}{\var}{(\stampi{\ind{\stamp_0}+\nolinebreak 1},\linebreak[1]\stamp_0+\dura{\act_0})}&{},\linebreak[1]
 \refinex{\auxz{\plan}}{\var}{(\stamp_0,\stamp_0)},
 \stamp_1\}&
 {}\setminus {}
 \\
 %\{\max(\{\stamp+\nolinebreak\dura{\act} \mid (\stamp,\act)\in\nolinebreak\auxz{\plan}\})\}
 \multicolumn{2}{@{}r@{}}{\{\stampi{2|\auxz{\plan}|}\}}
\end{array}
\end{equation}
for each $\var\in\pred{\act_1}\cup\pree{\act_1}\cup(\pres{\act_2}\setminus\effep{\act_1})$,
and
\begin{equation}\label{eq:proof:3}
 \refinex{\plan}{\var}{(\stamp_2,\stamp_2)}\in
 \{%\refinex{\plan}{\var}{(\stampi{\ind{\stamp_1}+1},\linebreak[1]\stamp_1+\dura   {\act_1})},
 \refinex{\auxz{\plan}}{\var}{(\stamp,\stamp)},\linebreak[1]
 \stamp_1,\linebreak[1]
 \stamp_1+\dura{\act_1}\}
 \setminus
 \{\stampi{2|\auxz{\plan}|}\}
\end{equation}
for 
\begin{equation}\label{eq:proof:4}
\stamp=\stampi{\min(\{1\leq\nolinebreak i\leq\nolinebreak 2|\auxz{\plan}| \mid \linebreak[1] \stamp_0+\dura{\act_1}\leq\stampi{i}\})}
\end{equation}
and
each $\var\in\pres{\act_2}$.
Finally, the conditions \refer{3}{c} and \refero{4}
in Definition~\ref{def:macro} yield that
\begin{equation*}
\begin{array}{@{}r@{}l@{}}
 \big(\pred{\act_2}\cup
 (\pree{\act_2}\setminus\pree{\act_0})\big)\setminus \pred{\act_0} & {}
 \subseteq
 {} \\
 \big(\effsp{\act_2}\cup\effep{\act_1}\setminus \effsm{\act_2}&\big) \cap
 {} \\
 \{\var\in\vars \mid \auxv{\neg\var}\in\nolinebreak\auxa{\act_0}&\}\text{,} % with
\end{array}
\end{equation*}
so that
\begin{equation}\label{eq:proof:5}
\begin{array}{@{\hspace*{-1pt}}r@{}r@{}l@{}}
 &\multicolumn{2}{@{}l@{}}{\refinex{\plan}{\var}{(\stampi{\ind{\stamp_2}+\nolinebreak 1},\linebreak[1]\stamp_2+\dura{\act_2})}
 \in {}}
 \\
 \{&
 \refinex{\auxz{\plan}}{\var}{(%\stampi{\ind{\stamp_0}+1}
% \stamp_0+\dura{\act_1}
 \stampi{\min(\{1\leq\nolinebreak i\leq\nolinebreak 2|\auxz{\plan}| \mid \linebreak[1] \stamp_0+\dura{\act_1}\leq\stampi{i}\})}
 &{},{}
 \\ &
 \stamp_0+\dura{\act_0})},
 \stamp_2,
 \stamp_1+\dura{\act_1}\}
 \setminus
 \{\stampi{2|\auxz{\plan}|}&\}
\end{array}
\end{equation}
for each $\var\in\pred{\act_2}\cup(\pree{\act_2}\setminus\pree{\act_0})$, and
\begin{equation}\label{eq:proof:6}
\begin{array}{@{\hspace*{-5.5pt}}r@{}l@{}l@{\hspace*{-2pt}}}
\refinex{\plan}{\var}{(\stamp_2+\dura{\act_2},\stamp_2+\dura{\act_2})}
& {} \in {}
\\
\{%\refinex{\plan}{\var}{(\stampi{\ind{\stamp_2}+1},\linebreak[1]\stamp_2+\dura{\act_2})},
  \refinex{\auxz{\plan}}{\var}{(\stamp_0+\dura{\act_0},\stamp_0+\dura{\act_0})}&{},
  \stamp_2,\linebreak[1]
  \stamp_1+\dura{\act_1}\}
& {} \setminus {}
\\
\multicolumn{2}{@{\hspace*{-5.5pt}}r@{}}{\{\stampi{2|\auxz{\plan}|}\}}
\end{array}
\end{equation}
for each $\var\in\pree{\act_2}$.

The conditions \eqref{eq:proof:1}--\eqref{eq:proof:6} show that $\plan$ is consistent and establish that the induction hypothesis holds for the refined
plan $\plan$ of $\auxs{\plan}$.
Moreover,
$\goal\subseteq\statei{\ind{\max(\{\stamp+\dura{\act} \mid (\stamp,\act)\in\plan\})}}$
holds because
$\statei{\ind{\auxs{\stamp}+\dura{\auxs{\act}}}}
 \subseteq
 \statei{\ind{\stamp'+\dura{\act'}}}$
for each $(\auxs{\stamp},\auxs{\act})\in\auxs{\plan}$
and $(\stamp',\act')\in\plan$ such that 
$\stamp'=\max(\{\stamp \mid (\stamp,\act)\in\plan,
                             \stamp+\dura{\act}\leq
                             \auxs{\stamp}+\dura{\auxs{\act}}\})$.
\end{proof}

\begin{corollary}\label{cor:macro}
Let % $\task = (\vars,\acts,\init,\goal)$ be a planning task, and
$\auxs{\plan}$ be a solution % of size $n-1$
for an effect-safe
task $\auxs{\task}=(\auxs{\vars},\linebreak[1]\auxs{\acts},\linebreak[1]\auxs{\init},\linebreak[1]\goal)$.
Then, 
there is some refined plan $\plan$ of $\auxs{\plan}$ for
$\task=(\vars,\linebreak[1]%\acts\cup
\{\act\mid (\stamp,\act)\in\plan\},\linebreak[1]\init,\linebreak[1]\goal)$
such that $\plan$ is a solution for $\task$ and
$\{\act\mid (\stamp,\act)\in\plan\}$ 
does not include any macro-action.
Moreover, for
any refined plan $\plan'$ of $\auxs{\plan}$
such that
$\{\act\mid (\stamp,\act)\in\plan'\}\subseteq
 \{\act\mid (\stamp,\act)\in\plan\}$,
% with
% $\{\act' \mid (\stamp',\act')\in\plan'\}\subseteq\acts$,
we have that $|\plan'|=|\plan|$, and
for each $1\leq i\leq 2|\plan|$,
$\seqj{i}$ is a starting (or ending) event for $\act$ iff
$\seqi{i}$ is a starting (or ending) event for~$\act$.\qed
\end{corollary}

\begin{proof}
All refined plans $\plan$ of $\auxs{\plan}$ obtained from
the base plan of $\auxs{\plan}$ by iteratively splitting
time-stamped (macro\nobreakdash-)actions
$(\stamp_0,\act_0)$ with $\act_0=\comp{\act_1}{\act_2}$
into $(\stamp_1,\act_1)$ and $(\stamp_2,\act_2)$
according to Definition~\ref{def:refine} yield
up to concrete time stamps identical sequences of time-stamped states
and are by Theorem~\ref{thm:refine} solutions
for $\task$.
\end{proof}

The crucial property guaranteed by solutions for an effect-safe task
is that volatile preconditions and effects signaled by mutex atoms
$\auxv{\neg\var}$ or $\auxv{\var}$, respectively, in $\auxa{\act}$
cannot be manipulated in uncontrolled ways during a macro-action~$\act$.
In particular, preconditions of concurrent actions must not rely on
$\var$ when $\auxv{\var}$ indicates that the atom $\var$ gets falsified
at some point within the duration of $\act$.
This does not necessitate $\var$ to stay excluded from add effects
applying when macro-actions are unfolded, yet any such effect is not exploited
to build and refine a baseline solution for the effect-safe task.

\begin{example}\label{ex:unfold}
Let $\act_1$ and $\act_2$ be macro-actions such that
$\auxv{\var}\in\auxa{\act_1}$ and $\var\in\effep{\act_2}$.
Then, the ending event for $\auxs{\act_2}$ cannot be placed in-between
the start and end of $\auxs{\act_1}$ in order to prevent that other actions
build on the volatile atom~$\var$.
However, when $\act_1$ and $\act_2$ are unfolded into
their constituent actions, $\var$ may well be included
in add effects falling within the duration of $\act_1$,
where temporarily enabling $\var$ does not compromise the
preconditions of any concurrent action.\qed
\end{example}

\section{Evaluation}\label{evaluation}

We evaluate the impact of sequential macro-actions on planning performance by applying three state-of-the-art planners to solve instances from four domains.
% In this section the tested domains are described, together with the metrics and the methodology used for evaluating our macro technique.
The first domain consists of a PDDL encoding of the RoboCup Logistics League (RCLL) along
with the instance collection used by \citeauthor{symp} (\citeyear{symp}) for assessing and comparing domain models with manually defined macros.
% The exact same domain and instances have already been used for an evaluation of different modeling technique involving macros in \cite{symp}.
In the RCLL domain, a team of three autonomous mobile robots cooperatively assembles % a set of 
products by interacting with production stations in a real-world environment. A characteristic property of the RCLL is that a robot can only perform one specific action at its current location before it has to move on. This makes the domain well-suited for macro-actions, as % . More specifically, 
\emph{move} actions can be merged with 17 different types of \emph{interaction} actions (to accomplish some production subtask at a given location). % from the original domain. 
The other three domains % belong to the official benchmark suit of 
originate from the International Planning Competition (IPC) \cite{vallati_chrpa_mccluskey_2018}: Road Traffic Accident Management (RTAM), Driverlog, and Satellite. %encode logistics tasks % scenarios involving 
%of agents that, like in the RCLL, move and perform transports between different locations, while the Satellite domain models maneuvering a satellite and controling its on-board tools.

The RTAM domain (unlike RCLL and Driverlog) associates specific tasks with separate types of agents: ambulances, police cars, and tow trucks. Since some of these agents may have to perform more than one task at a given location, we opted for a different approach to compose macro-actions than used for RCLL. Instead of creating macro-actions by merging a \emph{move} with an \emph{interaction} action, some of the latter actions are taken together. More specifically, \emph{first\_aid} and \emph{load\_victim} (to an ambulance) form a macro-action, and a second one is constructed by combining \emph{unload\_victim} and \emph{deliver\_victim} (to a hospital).

In Driverlog, a number of drivers need to walk to the positions of trucks, get on board, and drive them between different locations in order to load and unload packages. Differently from RCLL, in this domain it may not be possible to reach every target location by performing a single \emph{move} action. As a consequence, merging \emph{move} with \emph{interaction} actions and planning with the constructable macro-actions only risks to discard routes available in the original domain and may make problem instances infeasible.
However, this issue can be resolved by connecting each pair of locations by a \emph{move} action whose duration corresponds to a shortest path determined in a pre-processing step,
so that every target location becomes reachable by means of a single \emph{move}.
% In this way, every location becomes reachable with a single \emph{move} action.
Note that, for each plan based on this new domain, a corresponding plan for the original domain exists, and vice versa.
For comparability, we perform our Driverlog experiments with the pre-processed domain
incorporating shortest paths, which is also taken as basis for introducing macro-actions.
% the macro-version of the domain is compared against the version featuring the pre-processed shortest paths. 
% Many attempts of selecting the right pairs of actions to merge have been tried for this domain. The fact 
The macro-action composition is, however, still more sophisticated than for RCLL because the
possibility of performing several \emph{load} and \emph{unload} actions at the same location would induce an imbalance when merging them with \emph{move} actions.
Unlike that, the \emph{walk} action of a driver can readily be merged with \emph{board} as well as \emph{disembark}, leading to the two macro-actions \emph{walk\_board} and \emph{disembark\_walk}, where the second one may be needed once at the end of a plan in order to bring drivers to their goal locations.%
\footnote{We encountered the imbalance between \emph{move} actions for trucks and possible tasks to perform at a location, arising due to the specific packages included in instances, in our preliminary investigations, and did then only make limited use of shortest paths to build macros from \emph{walk} actions. However, other ideas to compose macro-actions may also take \emph{move} actions for trucks into account.}

The Satellite domain is rather simple by containing 5 actions only: \emph{turn\_to}, \emph{calibrate}, \emph{take\_image}, \emph{switch\_on}, and \emph{switch\_off}. The goals consist of taking pictures towards specific directions, after activating and calibrating the corresponding instruments. We here opt to use \emph{turn\_to\_calibrate} and \emph{turn\_to\_take\_image} as two macro-actions, considering that plans for the original domain frequently include these action pairs in sequence.
% . The justification for this choice is that such actions were often present sequentially inside the plans for the native domain.

\begin{figure*}
  \centering
  \includegraphics[width=0.95\textwidth]{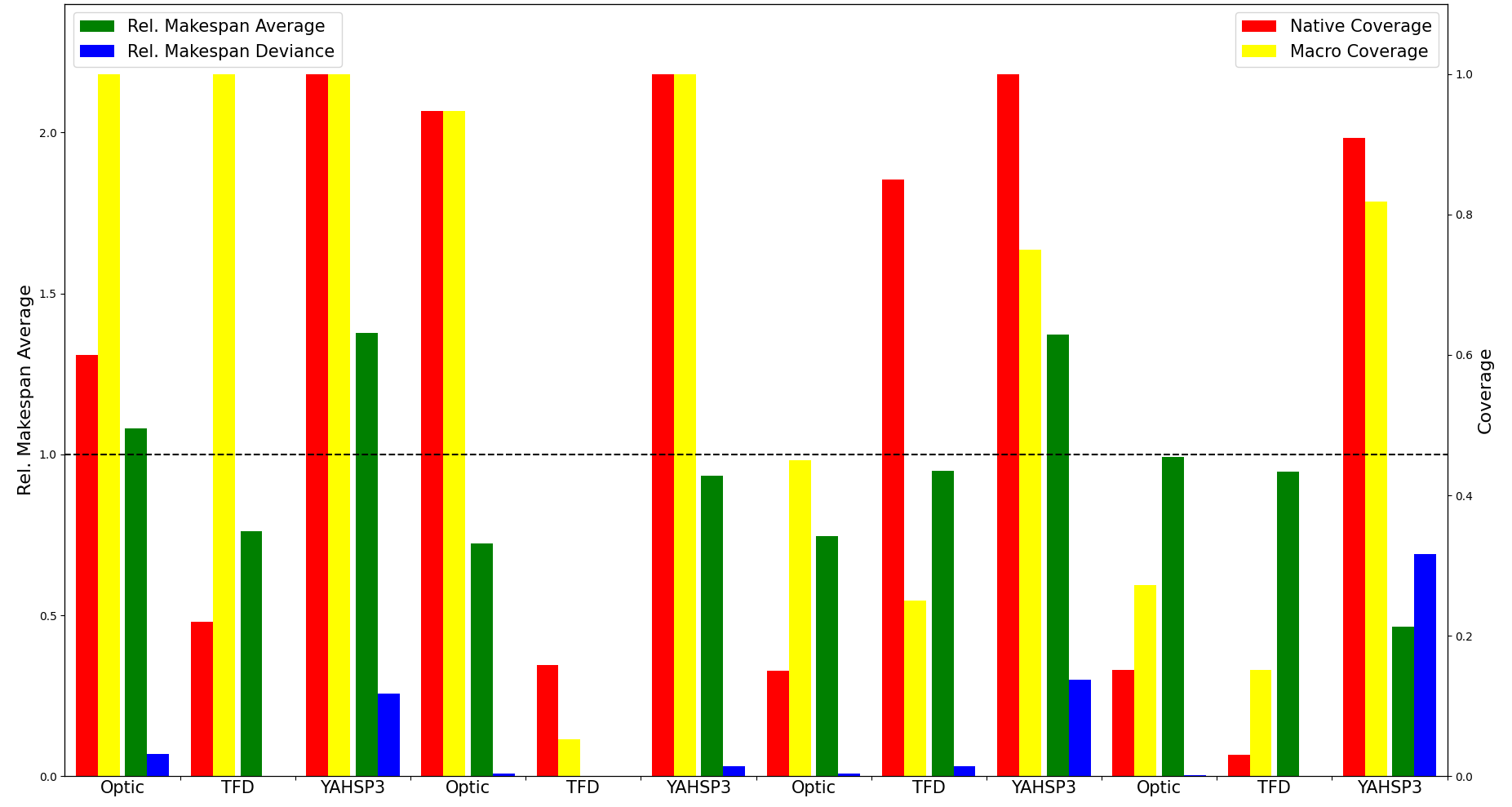}
  \caption{For each combination of domain and planner, the chart displays the Coverage and average Relative Makespan over the instance set.}
  \label{eval}
\end{figure*}%
  For each domain, the original actions composing the introduced macros are replaced by the macro-actions in order to improve the solving process, at the potential cost of losing optimality in case applying the ordinary actions off sequence permits plans to finish earlier.
  % Existing works on macros for classical planning \cite{ChrpaV22} show that replacing the original actions often leads to better performance by means of different metrics.
  Our comparison includes three state-of-the-art planners: the popular Optic system \cite{ICAPS124699}, also serving as baseline planner at the IPC 2018 edition, as well as the Temporal Fast Downward (TFD) \cite{eymaro09a} and YAHSP3 \cite{vidal2014yahsp3} planners, which achieved the runner-up and winner positions at the IPC 2014 edition \cite{vallati_chrpa_mccluskey_2018}. %and right after a portfolio planner (which included both of them) in IPC 2018. 
  The benchmark set consists of 50 instances for RCLL, 20 instances for both RTAM and Satellite, and 44 instances for Driverlog,
  where we run each planner for up to 15 minutes wall-clock time per instance 
  on a PC equipped with an Intel Core i5 10300h CPU and 16 GB RAM under Ubuntu 18.04,
  using either the original or the macro-action domain.
  Notably, the composition of macro-actions as specified in Section~\ref{macrodefinition} is automatically performed at the level of first-order PDDL domains by a Java tool we developed for this purpose.%
%  Given the native domains and the pairs of actions we want to merge, in order to automatically generates the macro-domains a tool has been developed in Java.%, and it is publicly available here
  \footnote{\url{https://gitlab.com/mbortoli/temporalmacro}}
  % The evaluation has been performed on a PC featuring an Intel Core i5 10300h CPU, 16 GB of RAM and Ubuntu 18.04.

  Figure~\ref{eval} indicates the original domain by ``Native'' and the new one replacing
  some of the ordinary actions by ``Macro''.
  The displayed metrics are Coverage, i.e., the ratio of instances for which some satisficing plan is obtained to the number of all instances in a domain, and average 
  Relative Makespan, comparing the finishing time of best plans found in 15 minutes between the original domain and the one with macro-actions.
  That is, the Relative Makespan considers instances such that a planner found at least one solution for either version of the domain,
  where values greater than $1$ express better plan quality for the domain with macros, or worse plan quality otherwise.
  In addition, we quantify the Relative Makespan deviance as an indicator of the plan quality differences due to the version of a domain, and greater values mean that the plan quality per instance varies significantly w.r.t.\ the (non-)use of macros.

As can be seen by surveying the 12 combinations of domains and planners in Figure~\ref{eval}, we achieve noticeable and relatively consistent improvements
in the Coverage metric, as anticipated in view of the reduced combinatorics due to macro-actions.
The Coverage improvements are particularly striking for Optic and TFD on the RCLL domain, where all ordinary actions of the native domain can be combined into macros that reduce the level of detail addressed in the planning process.
With the YAHSP3 planner, some satisficing plan is obtained for each RCLL instance, regardless whether macros are used, and it is generally geared towards quickly finding some satisficing plan, which leads to a higher Coverage than achieved by Optic and TFD.
In four cases, concerning TFD on RTAM and Satellite as well as YAHSP3 on Satellite and Driverlog, the use of macro-actions yields lower Coverage.
This observation shows that the effectiveness of introducing macros is relative to the planner as well as the domain under consideration, as replacing actions affects the planning process, e.g., pre-processing techniques and search heuristics. % applied by a planner.
While TFD does not seem to benefit from replacing some but not all of the ordinary actions in a domain by macros, Optic consistently benefits and obtains plans for the same number of instances or more.

% We achieved a consistent enhancement in the Coverage, as anticipated. It improved in 5 cases out of 12, sometimes by a considerable amount, like in the TFD-RCLL test. In 3 cases the Coverage for the two domain versions was the same. Among them, in 2 cases all the instances were solved, so an improving was not possible. In 4 cases out of 12, the Coverage was lower. Such cases mostly involved the TFD planner, which also did not benefit from the macro-actions in makespan. Without considering it, the Coverage improves in 4 cases out of 6, it remains the same in 1 case and it is worse in only 1 case (YAHSP3-Satellite) by a small margin.

Although the introduction of macro-actions aims at improving the planning performance without taking the finishing time of (still) feasible plans into account,
it can also be informative to check whether macro-actions allow
planners to find better plans in the given time limit of 15 minutes per run.
When inspecting the average Relative Makespan in Figure~\ref{eval},
where values greater than $1$ indicate shorter finishing times for plans obtained with macro-actions, we observe that the plan quality improves for 3 combinations of domains
and planners, and deteriorates for 8 of them, although they are often almost on par.
Here it is apparent that all of the planners compute longer (or no) plans for the macro version of the RTAM domain, i.e., the average Relative Makespan lies below $1$.
Beyond that, TFD yields worse plan quality with macro-actions for the RCLL, Satellite and Driverlog domains, and the same applies to Optic on Satellite.
We checked that, for TFD, such deterioration is mainly caused by the planner's pre-processing phase, which becomes % turns out to get 
much more time-consuming when macro-actions are introduced.
On the other hand, YAHSP3 manages to find plans of significantly better quality with macro-actions for two domains while being considerably worse in Driverlog, as also confirmed by the deviances taken over instances for which both the native and the macro domain lead to at least one solution.
That is, in some cases the reduced combinatorics due to macro-actions yields not only more satisficing but also better quality of plans obtainable in limited solving time, while such outcomes can hardly be predicted and require experimentation with a specific planner and domain.

Our experiments have shown that enhancing domains with macro-actions can change the landscape of heuristic features, having an impact on the quality of solutions and how fast planners can find them.
Summarizing the four investigated domains,
the Coverage on the RTAM domain drops for TFD only,
yet macro-actions deteriorate the Relative Makespan with YAHSP3 and Optic by a smaller or greater margin, respectively.
We conjecture that the reason for this resides in a much larger number of ordinary actions that remain in comparison to those replaced by macros.
% In fact, the two macro-actions we introduced only involve ambulances.
As a consequence, performance gains due to reducing the number of actions may be outweighed by the overhead induced by the mutex atoms associated with macro-actions. % , which have to be checked in the atomic actions as well.
The RCLL domain represents the opposite characteristic, as all ordinary actions of the native domain can be combined into macros, which
simplifies the planning process.
Hence, all compared planners benefit in terms of Coverage and find some solution for every instance, where YAHSP3 that obtains at least one solution per instance for either domain version also yields a better Relative Makespan.
The Satellite and Driverlog domains lie in-between the above two cases, and it depends on the particular planner and its solving techniques whether macros pay off (Optic) or the impact varies between domains (TFD and YAHSP3).

% which benefits the most from the use of macro-actions, both in terms of Coverage and Makespan. It is the only domain where all the atomic actions have been replaced by macro-actions, showing that the overhead due to the mutex predicates can be negligible if enough clever macro-actions are introduced. 

%Our conclusion is that macro-actions in temporal domain are particularly helpful if enough of them can be added to the domain, replacing the major part of atomic actions and generating a new macro domain which is considerably different from the original one. Many mutex predicates are shared between different macro-actions, so adding further macro-actions does not increase the number of mutex predicates (and the related overhead) by a significantly amount, often not even by any amount at all.

%%%% Macro-actions in both temporal and classical domains are particularly helpful if they are being frequently used in plans.
%%%% Such situation automatically applies if a lot of macro-actions can be added, replacing the major part of atomic actions and generating a new macro domain which is considerably different from the original one. Many mutex predicates are shared between different macro-actions, so adding further macro-actions either leaves the number of mutex predicates (and the related overhead) untouched or does not increase it  by a significantly amount.

\section{Related Work}\label{related}
Macro-actions are well-known in classical planning, starting with the STRIPS \cite{FikesN71} and REFLECT \cite{dawson1977role} systems in the 1970s. Classical planning systems may generate macro-actions in pre-processing or on the fly during the planning process.
%
%Many examples can be found in the literature regarding usage and definition of macro-actions in classical planning. However, the same does not apply for temporal planning, where action concurrency and time constraints makes this operation not naive. Definition of macro-actions in classical planning is trivial and can be achieved by merging the preconditions and effects of the involved atomic actions, together with a few adjustments. For instance, by keeping only the last effect in the sequence in case of contradicting effects, or by omitting preconditions which are internally achieved by an effect of a previous atomic action within the same macro.
%As a consequence, most of the works in the literature focuses on automatic techniques selecting the actions to merge into a macro, either as a pre-processing step which modifies the original domain, or as a solving heuristics. 
An example of the latter is the Marvin planner \cite{coles2005inference,coles2007marvin}, which determines macro-actions % on-the-fly during the planning process as
in the course of a plateaux-escaping technique for the well-known FF planner \cite{HoffmannN01}. %However, the majority of the work involve some pre-processing or training operation. 
The MACRO-FF planner \cite{botea2005macro} features two different pre-processing techniques to generate macros in classical planning, Component Abstraction (CA) and Solution (SOL): CA searches in the space of potential macros and identifies promising ones based on specific rules, while SOL extracts macros from sample plans by analyzing causal links between actions. \citeauthor{newton2007learning} (\citeyear{newton2007learning}) 
propose a technique to learn planner-independent macro-actions by means of genetic algorithms. \citeauthor{hofmann2017initial} (\citeyear{hofmann2017initial}) use Map Reduce to search through a plan database for macro-actions. Their work has recently been extended to support ADL features \cite{HofmannNL20}. %The Reflect planner \cite{dawson1977role} generates macro through a pre-analysis of the domain and the problem together. Wizard \cite{newton2007wizard} learns macro exploring the MACRO-SAT space and analyzing the plans of small problems. Macro are then further refined during the planning process. 
\citeauthor{miura2017automatic} (\citeyear{miura2017automatic}) present a technique to derive axioms reducing the number of (explicit) actions in the planning process. 
MUM \cite{chrpa2014mum} is another system able to learn macro-actions from sample plans such that the number of instances of learned macro-actions is minimized. Recent works \cite{chrpa2019improving,ChrpaV22} concern the generation of macro-actions by observing ``critical sections'' over ``lockable'' resources in action sequences. 

In contrast to classical planning,
very few works consider % have been published regarding the conception of 
macro-actions in the context of temporal planning. \citeauthor{wullinger2008spanning} (\citeyear{wullinger2008spanning}) present a technique to generate macro-actions out of partially overlapping temporal actions. Their macro-actions are, however, not guaranteed to be executable and must be filtered through a separate consistency check. To our knowledge, the most recent approach to define temporal macro-actions stems from a master thesis \cite{hansson2018temporal}, where procedures for the construction of sequential and parallel macro-actions, with support for numeric fluents, are provided. % However, a formal definition is missing, together with some key steps of the algorithm concerning the definition of mutex locks. Moreover, the thesis violates the semantics of PDDL 2.1 presented in \cite{foxlon03a}, by swapping the order of evaluation of postconditions and posteffects. 
The used model of durative actions deviates from PDDL 2.1 \cite{foxlon03a} by swapping the order of evaluating preconditions and applying effects at the end of actions.
%We chose to not test the performance of macros operations built with such techniques due to the number of flaws and lacks in the procedure, avoiding the risk of recreating an algorithm which differs from the one that the authors had in mind. However, 
Concerning the (informally presented) macro-action construction, mutex locks are introduced in virtually all cases where internal events may change the values of fluents during a macro-action.
This is different from our, less restrictive macro-action concept, which permits compatible effects 
of concurrent actions to fall within the duration of a macro-action.
% On top of this statement, some key differences w.r.t.\ the definition given in this paper can be stated. Given a set of mutex locks, \citeauthor{hansson2018temporal} (\citeyear{hansson2018temporal}) explains that every other action which ``uses'' the corresponding atoms should acquire the mutex, in order to avoid concurrent execution. While it is clear that the term ``use'' refers to effects involving such an atom, we are not sure if preconditions are considered as well. If they are not considered, invalid plans may be generated. If not, the mutex condition become very strict, cutting out many potential solutions.
To illustrate the diverging ideas, consider the composition of two actions $\act_1$ and $\act_2$ 
with $\var_1\in\effsp{\act_1}$, $\var_2\in\effep{\act_1}\cup\effsp{\act_2}$, and $\var_3\in\pree{\act_2}$ into a macro-action $\act_0=\comp{\act_1}{\act_2}$, along with an ordinary action $\act$ such that $\var_1\in\pres{\act}$, $\var_2\in\effsp{\act}\cup\effep{\act}$, and $\var_3\in\effep{\act}$.
Then, $\act_0$ needs to start before $\act$ to enable the precondition $\var_1$, while the end event for $\act$ in turn enables the precondition $\var_3$ required at the end of $\act_0$.
Such concurrent execution is admitted by our macro-action concept, as $\auxv{\neg\var_2}\in\auxa{\act_0}$ only rejects a cancellation of the internal add effect on~$\var_2$.
Unlike that, locking $v_2$ by a mutex according to \citeauthor{hansson2018temporal} (\citeyear{hansson2018temporal}) rules out that $\act_0$ and $\act$ are executed in parallel,
so that neither the macro-action $\act_0$ nor the ordinary action $\act$ is applicable.

% To fix the idea, consider the following example of required concurrency. Let's say a macro is obtained by merging the actions $a_1$ and $a_2$, and we have another action $a_3$. Given atoms $v,u,w$, we now encode required concurrency, with  $v \in \effs{a_1},  v \in \pres{a_3}, w \in \effe{a_3}, w \in \pree{a_2}$. If it is the case that the macro adds an atom $u$ in the middle which is also added by $a_3$, such that $u \in \effs{a_2}$ and $u \in \effs{a_3}$ or $u \in \effe{a_3}$, the two actions can no longer be scheduled in parallel, resulting in non-existence of a plan. Since the definition presented in this paper uses two different types of mutex locks, one for enabling and one for disabling, such a situation can be modeled without compromising the satisfiability of the problem.

Macro-actions and abstractions also find application in languages and paradigms beyond PDDL. The Modular Action Description language \cite{lifschitz2006modular} aims to create a database of general-purpose knowledge about actions, which can be referred to for defining a new specific action. \citeauthor{fadel2002planning} (\citeyear{fadel2002planning}) presented a method to use complex actions as planning operators in the situation calculus, together with a compiler to generate the corresponding PDDL code. \citeauthor{gabaldon2002programming} (\citeyear{gabaldon2002programming}) introduced the concept of Hierarchical Task Networks in the situation calculus by defining hierarchical complex actions in GonGolog.
Moreover, \citeauthor{banihashemi2017abstraction} (\citeyear{banihashemi2017abstraction}) define a general framework for agent abstraction, where high-level domains described in the situation calculus are mapped to ConGolog programs and low-level formulas.

\section{Conclusion and Future Work}\label{conclusion}
 %Logistics domains, which involve concurrent execution of actions and usage of resources, represent a good real world scenario for the application of temporal planning. However, the high computational complexity of temporal planning remains the main obstacle to a successfully application of this technique to many suitable domains. A well-known technique in classical planning to increase the performance of the planning process is the definition of macro-actions. 

Temporal planning, which involves concurrent execution of actions and sharing of resources, allows for modeling and solving a variety of planning and scheduling tasks. However, the high computational complexity of temporal planning remains a notorious obstacle for its successful application to challenging target domains. A popular approach in classical planning to reduce combinatorics and boost the performance of the planning process is the introduction of macro-actions.

 In this paper, we propose a general concept of sequential macro-actions for temporal planning that guarantees the applicability of plans. Sequential macro-actions are particularly advantageous in logistics domains, where it is common that the activities of agents follow specific patterns.
 Our experiments investigate the performance of three state-of-the-art planners on four domains (out of which three are logistics-related).
 For the majority of tested planners and domains, more satisficing plans and in some cases also better plan quality are obtained when frequent sequences of ordinary actions are encapsulated and replaced by a macro.
 In fact,
 while native domains always admit solutions that are at least as good as a plan with macro-actions, enhancing temporal domains by macro-actions can sometimes help to guide planners to suitable solutions in shorter solving time. % , consequently improving the coverage.
 This is particularly the case when the macro-actions subsume and replace a large portion of ordinary actions, which is not unlikely for logistics domains.  
 % After the formalization of sequential macro-actions and the description of the property of suitable domains, four domains has been evaluated with three different planners, including logistics domains and domains from the IPC, comparing the original models and those enhanced by macro-actions. As can be seen from the results, the use of macro-actions leads to better solvability for most of the domains and planners tested. This is particularly true when a good number of suitable macro-actions are introduced, due to the computational overhead caused by the need of auxiliar atoms.
 However, our macro-action concept is not exclusive to logistics domains and can be 
 % successfully applied to other problems as well
 applied to any temporal planning task, yet the substitution of ordinary actions by macro-actions needs to be done with care to preserve satisficing or optimal plans, respectively.
 
% We believe that temporal macro-actions can be also applied to other kinds of domains. 
As a part of future work, we want to investigate methods to automatically detect suitable candidates for macro-actions in a given domain. Moreover, the formalization of further kinds of macro-actions in temporal planning, like parallel or, more generally, overlapping macro-actions and support for numeric fluents, constitutes an interesting future direction.

\section*{Acknowledgements}

M.\ De Bortoli and M.\ Gebser were funded by Kärntner Wirtschaftsförderungs Fonds (project no.\ 28472), cms electronics GmbH, FunderMax GmbH, Hirsch Armbänder GmbH, incubed IT GmbH, Infineon Technologies Austria AG, Isovolta AG, Kostwein Holding GmbH, and Privatstiftung Kärntner Sparkasse.
L.\ Chrpa was funded by the Czech Science Foundation (project no.\ 23-05575S).
M.\ De Bortoli's and M.\ Gebser's visit to CTU in Prague was funded by the OP VVV project no. EF15\_003/0000470 ``Robotics 4 Industry 4.0'' and by the Czech Ministry of Education, Youth and Sports under the Czech-Austrian Mobility programme (project no.\ 8J22AT003), respectively.
L.\ Chrpa's visits to University of Klagenfurt were funded by OeAD, Austria's Agency for Education and Internationalisation (project no.\ CZ 15/2022).

%% The file kr.bst is a bibliography style file for BibTeX 0.99c
\bibliographystyle{kr}
\bibliography{aaai23}

\end{document}